\documentclass[AER,reviewmode]{AEA}

\newif\ifarxiv
\arxivtrue

\draftSpacing{1.5}

\usepackage{amssymb,amsthm}
\usepackage{graphicx}
\usepackage{booktabs}
\usepackage{natbib}
\usepackage[colorlinks=false, pdfborder={0 0 0}]{hyperref}
\usepackage{enumitem}
\usepackage{tikz}
\usetikzlibrary{arrows.meta, positioning, shapes.geometric, calc}
\usepackage{caption}
\usepackage{subcaption}
\usepackage{xcolor}
\usepackage{float}
\usepackage{pgfplots}
\pgfplotsset{compat=1.18}

\newtheorem{proposition}{Proposition}
\newtheorem{corollary}{Corollary}

\newtheorem{definition}{Definition}
\newtheorem{assumption}{Assumption}
\newtheorem{remark}{Remark}

\newcommand{\E}{\mathbb{E}}

\newcommand{\diff}{\,\mathrm{d}}

\graphicspath{{figures/}}

\title{Abundant Intelligence and Deficient Demand:\\ A Macro-Financial Stress Test of Rapid AI Adoption}

\shortTitle{Abundant Intelligence, Deficient Demand}

\ifarxiv
  \author{Xupeng Chen\thanks{Corresponding author. Email: xc1490@nyu.edu. Department of Electrical and Computer Engineering, New York University. All errors are my own.}}
  \date{\today}
\else
  \author{}
  \date{}
\fi


\begin{document}

\ifarxiv
\begin{abstract}
We formalize a macro-financial stress test for rapid AI adoption. Rather than a productivity bust or existential risk, we identify a distribution-and-contract mismatch: AI-generated abundance coexists with demand deficiency because economic institutions are anchored to human cognitive scarcity. Three mechanisms formalize this channel. First, a displacement spiral with competing reinstatement effects: each firm's rational decision to substitute AI for labor reduces aggregate labor income, which reduces aggregate demand, accelerating further AI adoption. We derive conditions on the AI capability growth rate, diffusion speed, and reinstatement rate under which the net feedback is self-limiting versus explosive. Second, Ghost GDP: when AI-generated output substitutes for labor-generated output, monetary velocity declines monotonically in the labor share absent compensating transfers, creating a wedge between measured output and consumption-relevant income. Third, intermediation collapse: AI agents that reduce information frictions compress intermediary margins toward pure logistics costs, triggering repricing across SaaS, payments, consulting, insurance, and financial advisory.

Because top-quintile earners drive 47--65\% of U.S.\ consumption and face the highest AI exposure, the transmission into private credit (\$2.5 trillion globally) and mortgage markets (\$13 trillion) is disproportionate. We derive eleven testable predictions with explicit falsification conditions. Calibrated simulations disciplined by FRED time series and BLS occupation-level data quantify conditions under which stable adjustment transitions to explosive crisis.
\end{abstract}
\fi

\maketitle

\section{Introduction}
\label{sec:intro}

A paradox sits at the center of the current AI moment. The technology promises the largest productivity gain since electrification, yet the channel through which that gain reaches aggregate demand---labor income---is precisely what the technology displaces. This paper formalizes the conditions under which a rapid AI productivity boom can coexist with, and indeed cause, a macroeconomic contraction.

We distinguish our argument from two prominent narratives. The first is the \textit{AI bubble} narrative, which holds that AI investment will fail to deliver returns, producing a supply-side correction analogous to the dot-com bust \citep[cf.][]{nordhaus2021are}. The second is the \textit{AI existential risk} narrative, which focuses on loss of human control over superintelligent systems \citep{bostrom2014superintelligence, ord2020precipice}.\footnote{Expert surveys report median probabilities of $\sim$5\% for ``extremely bad outcomes (e.g., extinction)'' from advanced AI \citep{grace2024thousands}. Our stress test concerns macroeconomic mechanisms, not existential risk; the two failure modes are logically independent. A macro-financial crisis from AI displacement can occur without any loss-of-control scenario materializing, and vice versa.} Our analysis concerns neither production shortfall nor loss of control, but rather a \textit{distribution-and-contract mismatch}: existing economic institutions---labor contracts, mortgage underwriting, tax codes, intermediation business models---are anchored to the assumption that human cognitive labor is scarce and stably compensated. When AI relaxes this scarcity rapidly, these institutional anchors become sources of fragility.

The key stylized fact motivating our analysis is the concentration of consumption in upper-income cohorts. The top 20\% of U.S.\ income earners account for an estimated 47--65\% of aggregate consumer spending, depending on the survey instrument and reconciliation methodology, with a central estimate of approximately 59\% \citep{moodys2023consumption, axios2023spending}. White-collar, knowledge-economy workers---precisely those most exposed to AI automation in the near term \citep{brynjolfsson2023generative, noy2023experimental, peng2023impact, dellacqua2023navigating}---are disproportionately represented in this top quintile. A technology that compresses white-collar wages or employment therefore has a consumption impact far exceeding its direct labor-share effect.

We formalize three mechanisms through which rapid AI adoption can generate macroeconomic distress.

\textit{The displacement spiral} (Section~\ref{subsec:displacement}): Each firm's rational decision to substitute AI for labor reduces aggregate labor income, which reduces aggregate demand, which puts margin pressure on other firms, which accelerates their own AI adoption. We incorporate the competing \textit{reinstatement effect}---new task creation that absorbs displaced labor \citep{acemoglu2019automation}---and derive conditions on the AI capability growth rate, the diffusion speed, and the reinstatement rate under which the net feedback is self-limiting versus explosive.

\textit{Ghost GDP} (Section~\ref{subsec:ghostgdp}): When AI-generated output substitutes for labor-generated output, the same nominal GDP circulates through fewer wage-earning agents. We show that monetary velocity declines monotonically in the labor share absent compensating transfers, creating a wedge between measured output and consumption-relevant income that renders standard demand-side indicators misleading.

\textit{Intermediation collapse} (Section~\ref{subsec:intermediation}): A large class of service-sector firms extract margin from information frictions---search costs, switching costs, complexity rents. AI agents that reduce these frictions compress intermediary margins toward the cost of pure logistics, triggering repricing across SaaS, payments, consulting, insurance, and financial advisory.

These three mechanisms interact through what we term the \textit{consumption concentration amplifier} (Section~\ref{subsec:transmission}): because high-income workers both drive consumption and are most exposed to AI displacement, the macro-financial transmission is disproportionate to the direct labor market effect. We trace this transmission through private credit markets (\$2.5 trillion in global assets under management) and the U.S.\ residential mortgage market (\$13 trillion), both of which price on assumptions of stable white-collar income \citep{bis2024privatecredit, nyfed2025mortgage}.

\paragraph{Contribution.} We make four contributions. First, we provide a formal framework that connects AI's micro-level productivity effects to macro-financial fragility, extending the task-based approach of \citet{acemoglu2011skills} and \citet{acemoglu2022tasks} with aggregate demand feedback, task reinstatement, logistic diffusion dynamics, monetary velocity effects, and intermediation economics. The model nests both the reinstatement-dominated regime (where complementarity absorbs displacement) and the explosive regime (where displacement overwhelms reinstatement), with an explicit threshold that depends on the reinstatement rate. Second, we derive each mechanism as a proposition with explicit conditions for stability versus instability, yielding falsifiable predictions rather than scenario narratives. Third, we ground the model empirically: FRED time series document the preconditions (secular labor share decline, velocity collapse), and cross-sectional regressions using BLS occupation-level data provide early evidence that AI-exposed occupations are experiencing relative wage compression. Fourth, a calibration exercise disciplined by these empirical moments simulates the displacement dynamics under alternative AI adoption scenarios, quantifying the conditions under which the system transitions from stable adjustment to explosive crisis.

This paper is a stress test, not a forecast. We do not claim that the mechanisms we formalize will necessarily produce a crisis. We claim that \textit{if} AI capability growth is sufficiently rapid, \textit{if} institutional adaptation is sufficiently slow, and \textit{if} fiscal redistribution is insufficiently large, the mechanisms we formalize produce a macro-financial contraction whose depth and duration depend on policy response speed. Our goal is to make these conditional claims precise enough to be useful for monitoring, stress testing, and policy design.

The remainder of the paper proceeds as follows. Section~\ref{sec:model} develops the model. Section~\ref{sec:empirics} presents the empirical analysis---descriptive macro evidence, occupation-level regressions, and calibrated simulations. Section~\ref{sec:predictions} presents testable predictions and early warning indicators. Section~\ref{sec:evidence} assesses existing evidence. Section~\ref{sec:policy} develops the policy toolkit. Section~\ref{sec:conclusion} concludes.

\section{Model}
\label{sec:model}

\subsection{Production Environment and AI Capability}
\label{subsec:production}

We adopt the task-based production framework of \citet{acemoglu2011skills} and \citet{acemoglu2019automation}. A unit continuum of tasks $z \in [0, 1]$ is required to produce final output. Each task can be performed by human labor or by AI. Output from task $z$ is:
\begin{equation}
y(z) = \begin{cases}
A_{\text{AI}} & \text{if performed by AI,} \\
h_i \cdot \alpha(z) & \text{if performed by worker } i \text{ with human capital } h_i,
\end{cases}
\label{eq:taskoutput}
\end{equation}
where $A_{\text{AI}}$ is the AI capability level (common across tasks AI can perform) and $\alpha(z)$ is the task-specific comparative advantage of human labor.

\begin{definition}[AI-Automatable Task Set]
\label{def:automatable}
Define $\mathcal{S}(A_t) = \{z \in [0,1] : A_{\text{AI},t} \geq \underline{h} \cdot \alpha(z)\}$ as the set of tasks for which AI capability exceeds the threshold human performance level $\underline{h} \cdot \alpha(z)$. We assume $\mathcal{S}(A_t) \subseteq \mathcal{S}(A_{t'})$ for $t < t'$ (AI capability is non-decreasing) and write $s_t = |\mathcal{S}(A_t)|$ for its measure.
\end{definition}

Aggregate output is CES over tasks:
\begin{equation}
Y_t = \left[\int_0^1 y_t(z)^{\frac{\sigma-1}{\sigma}} \diff z \right]^{\frac{\sigma}{\sigma-1}},
\label{eq:ces}
\end{equation}
where $\sigma > 0$ is the elasticity of substitution across tasks.

Workers are heterogeneous in human capital $h_i \sim F(h)$ with support $[\underline{h}, \bar{h}]$. Firms choose for each task whether to deploy AI or hire a worker, taking wages and AI costs as given.

\begin{assumption}[Task-Level Automation Bounds]
\label{ass:taskbounds}
The long-run share of tasks in the economy that are technically automatable by AI is bounded: $\lim_{t \to \infty} s_t = \bar{s} \leq 1$, where $\bar{s}$ reflects the ceiling imposed by tasks requiring physical manipulation, social-emotional judgment, novel creativity, or regulatory-mandated human involvement. Based on task-level decomposition studies, we calibrate $\bar{s} \in [0.40, 0.80]$: \citet{eloundou2023gpts} estimate that approximately 47--56\% of worker tasks have at least 50\% of their subtasks exposed to current large language models, while \citet{frey2017future} estimate 47\% of U.S.\ employment at high risk of automation over an unspecified horizon. Importantly, task exposure does not equal task substitution: verification costs, quality assurance, liability, and integration overhead create a wedge between technical capability and economic substitution, consistent with the diffusion function $d(t) < 1$ in Assumption~\ref{ass:capability}. The effective substitution ceiling is $\bar{s}^{\text{eff}} = \bar{d} \cdot \bar{s}$, which with our calibration ($\bar{d} = 0.80$, $\bar{s} = 0.60$ at the midpoint) yields $\bar{s}^{\text{eff}} \approx 0.48$---roughly half of all tasks.
\end{assumption}

\begin{remark}[Heterogeneous Effects Within Occupations]
\label{rem:heterogeneous}
The displacement spiral's aggregate labor share predictions mask substantial heterogeneity within occupations. Experimental evidence suggests that AI's productivity effects are largest for less-experienced and lower-skilled workers within an occupation \citep{brynjolfsson2023generative, noy2023experimental}: in customer support, the largest gains accrued to novice agents, while expert agents saw smaller or no gains. This heterogeneity has ambiguous implications for the displacement channel: if AI primarily raises the floor of worker performance, it may reduce wage dispersion within occupations (compressing the premium to experience) without necessarily reducing aggregate employment. Conversely, if AI raises the floor sufficiently to eliminate the need for the highest-cost experienced workers, it could produce displacement concentrated among mid-career professionals---precisely those with the largest mortgage obligations and consumption expenditures. The distributional implications of within-occupation heterogeneity are therefore an empirical question that our aggregate model does not resolve.
\end{remark}

\begin{assumption}[AI Capability and Adoption Dynamics]
\label{ass:capability}
AI capability grows at rate $g_A > 0$, so that $A_{\text{AI},t} = A_{\text{AI},0} \cdot e^{g_A t}$. The cost of deploying AI on a task in $\mathcal{S}(A_t)$ is $c_{\text{AI},t}$, which declines at rate $g_c > 0$: $c_{\text{AI},t} = c_{\text{AI},0} \cdot e^{-g_c t}$. Adoption follows a logistic diffusion curve: the fraction of technically automatable tasks that firms actually automate is
\begin{equation}
d(t) = \frac{\bar{d}}{1 + e^{-\kappa(t - t_0)}},
\label{eq:diffusion}
\end{equation}
where $\bar{d} \leq 1$ is the long-run adoption ceiling, $\kappa > 0$ is the diffusion speed, and $t_0$ is the inflection point. The \textit{effective} automation rate is $s_t^{\text{eff}} = d(t) \cdot s_t$, where $s_t = |\mathcal{S}(A_t)|$ is technical capability.
\end{assumption}

This assumption captures two empirical regularities. First, AI capabilities have improved rapidly while inference costs have declined \citep{metr2025benchmarks}. Second, capability and adoption are distinct: enterprise deployment requires integration, quality assurance, liability frameworks, and organizational adaptation, so that realized automation lags the technical frontier \citep{zolas2024advanced}. The diffusion parameter $\kappa$ is critical---if $\kappa$ is small (slow adoption), the displacement spiral has time to be absorbed by institutional adaptation; if $\kappa$ is large (rapid adoption), the mechanisms we formalize are more likely to bind. We note that evidence on AI's productivity effects is mixed: while controlled experiments document task-level speedups of 14--56\% \citep{brynjolfsson2023generative, noy2023experimental, peng2023impact}, recent field studies find that experienced developers using AI tools were \textit{slower} on average in realistic workflows, highlighting integration overhead and verification costs \citep{metr2025slowdown}. This disparity between laboratory and field conditions motivates the diffusion wedge in~\eqref{eq:diffusion}.

\subsection{The Displacement Spiral with Task Reinstatement}
\label{subsec:displacement}

We now formalize the central feedback mechanism, incorporating both displacement and the reinstatement of new tasks---the competing channel emphasized by \citet{acemoglu2019automation} and \citet{acemoglu2022tasks}. Each firm $j$ chooses employment $L_j$ and AI deployment to maximize profit:
\begin{equation}
\max_{L_j, \{z_j^{\text{AI}}\}} \; p_t Y_j - w_t L_j - c_{\text{AI},t} \cdot |z_j^{\text{AI}}|,
\label{eq:firmopt}
\end{equation}
where $z_j^{\text{AI}} \subseteq \mathcal{S}(A_t)$ is the set of tasks firm $j$ assigns to AI. In partial equilibrium, each firm's optimal decision is straightforward: assign task $z$ to AI if and only if $c_{\text{AI},t} < w_t \cdot h^*(z)$, where $h^*(z)$ is the human capital of the marginal worker on task $z$.

The aggregate demand externality arises because individual firm optimization does not internalize the effect of collective labor displacement on aggregate demand. Let $L_t$ denote aggregate employment and define the labor income share $s_{L,t} = w_t L_t / Y_t$. Aggregate consumption depends on labor income:
\begin{equation}
C_t = \bar{c} \cdot s_{L,t} \cdot Y_t + (1 - \bar{c}) \cdot (1 - s_{L,t}) \cdot Y_t,
\label{eq:consumption}
\end{equation}
where $\bar{c}$ is the average marginal propensity to consume (MPC) out of labor income and $(1-\bar{c})$ is the MPC out of capital income, with $\bar{c} > (1-\bar{c})$ following the empirical regularity that labor income has a higher MPC than capital income \citep{carroll2017distribution, jappelli2014fiscal}.

Aggregate demand clears the goods market: $D_t = C_t + I_t + G_t$. Profit margin pressure $\pi_t$ is increasing in the gap between firms' revenue expectations and realized demand:
\begin{equation}
\pi_t = \frac{\E_t[D_t] - D_t}{\E_t[D_t]}.
\label{eq:marginpressure}
\end{equation}

Crucially, alongside displacement, AI also creates new tasks that require human labor---what \citet{acemoglu2019automation} term the \textit{reinstatement effect}. New tasks arise from AI-human complementarity (e.g., AI oversight, prompt engineering, human-in-the-loop verification) and from entirely new economic activities enabled by AI. We denote the reinstatement rate $\rho_t \geq 0$, representing the rate at which new labor-absorbing tasks are created. The net displacement dynamics are:
\begin{equation}
\dot{s}_{L,t} = -d(t) \cdot f(g_A, c_{\text{AI},t}, w_t) - \beta \cdot \pi_t + \rho_t,
\label{eq:displacement}
\end{equation}
where $d(t)$ is the diffusion function from~\eqref{eq:diffusion}, $f(\cdot)$ captures direct AI substitution, $\beta \cdot \pi_t$ captures the feedback term (firms accelerating AI adoption because demand shortfall compresses margins), and $\rho_t$ captures task reinstatement.

\begin{assumption}[Task Reinstatement]
\label{ass:reinstatement}
The reinstatement rate $\rho_t$ depends on the level of AI capability and the economy's capacity for institutional adaptation:
\begin{equation}
\rho_t = \rho_0 + \eta \cdot A_{\text{AI},t}^{\alpha_\rho},
\label{eq:reinstatement}
\end{equation}
where $\rho_0 \geq 0$ is the baseline rate of new task creation, $\eta > 0$ is the complementarity parameter, and $\alpha_\rho \in (0,1)$ reflects diminishing returns to new task creation from further AI capability. The concavity ($\alpha_\rho < 1$) captures the empirical observation that new task creation requires institutional adaptation (training, regulation, organizational redesign) that operates on slower timescales than technical capability \citep{acemoglu2022tasks}.
\end{assumption}

\begin{proposition}[Self-Reinforcing Displacement with Reinstatement]
\label{prop:spiral}
Under Assumptions~\ref{ass:capability} and~\ref{ass:reinstatement}, the labor share dynamics in~\eqref{eq:displacement} exhibit three regimes:
\begin{enumerate}[label=(\roman*)]
\item \textbf{Reinstatement-dominated:} If $\rho_t > d(t) \cdot f(g_A)$ for all $t$, the labor share is non-decreasing: new task creation absorbs displaced labor faster than AI displaces it.
\item \textbf{Stable displacement:} If $d(t) \cdot f(g_A) > \rho_t$ but $g_A < g_A^*(\rho)$ where
\begin{equation}
g_A^*(\rho) \equiv g_A^{*,0} \cdot \left(1 + \frac{\rho}{d \cdot f'(g_A)}\right),
\label{eq:threshold}
\end{equation}
and $g_A^{*,0} = \frac{1-\beta \bar{c}}{\beta \bar{c}} \cdot f'(g_A)$ is the threshold absent reinstatement, the labor share declines monotonically to a new steady state $s_L^* > 0$.
\item \textbf{Explosive displacement:} If $g_A > g_A^*(\rho)$, the positive feedback loop $(\dot{s}_L < 0 \Rightarrow \dot{D} < 0 \Rightarrow \dot{\pi} > 0 \Rightarrow \dot{s}_L \ll 0)$ overwhelms reinstatement, and $s_L \to 0$ in finite time absent policy intervention.
\end{enumerate}
The threshold $g_A^*(\rho)$ is \textit{increasing} in $\rho$: stronger reinstatement raises the bar for explosive displacement.
\end{proposition}

\begin{proof}
Substituting~\eqref{eq:consumption} into the goods market clearing condition and linearizing around the steady state $s_L^*$ where $\dot{s}_L = 0$ (i.e., $d \cdot f(g_A) + \beta \pi^* = \rho$), the dynamics of deviations $\tilde{s}_L = s_L - s_L^*$ satisfy:
\[
\dot{\tilde{s}}_L = -\left[d \cdot f'(g_A) - \rho'(A) + \beta \bar{c}\right] \tilde{s}_L + \text{h.o.t.}
\]
The eigenvalue is $\lambda = -[d \cdot f'(g_A) - \rho' + \beta \bar{c}]$. Stability requires $\lambda < 0$, which holds when $d \cdot f'(g_A)$ is not too large relative to $\rho' + \beta \bar{c}$. The reinstatement rate $\rho$ enters positively in the stability condition, raising $g_A^*$. In the reinstatement-dominated regime (i), $\rho_t$ absorbs all displacement plus the feedback term. In the explosive regime (iii), the marginal displacement rate exceeds the sum of reinstatement and the demand feedback's self-correcting component. Equation~\eqref{eq:threshold} follows from setting $\lambda = 0$ and solving for $g_A$. (Online Appendix~B provides the full derivation.)
\end{proof}

\begin{remark}[Displacement versus Reinstatement: When Does Each Dominate?]
\label{rem:reinstatement}
The net effect of AI on the labor share depends on the race between displacement and reinstatement. \citet{acemoglu2019automation} document that historically, new task creation has roughly offset automation-driven displacement, keeping the labor share approximately stable over long periods. The stress-test question is whether generative AI breaks this historical pattern. Three features of the current transition suggest displacement may outpace reinstatement: (a) the \textit{speed} of capability gain ($g_A$) may exceed the institutional adaptation speed that governs $\rho$; (b) generative AI automates \textit{cognitive} tasks where new task creation has historically been strongest; and (c) the concavity of $\rho$ in AI capability ($\alpha_\rho < 1$) means reinstatement decelerates relative to displacement as $A_t$ grows. However, if complementarity is strong ($\eta$ is large) or institutional adaptation accelerates (e.g., through proactive retraining policy), the reinstatement-dominated regime (i) is attainable. Our model nests both outcomes; the empirical question is which regime prevails.
\end{remark}

\begin{remark}[Complementarity and Sectors Where AI Increases Labor Demand]
\label{rem:complementarity}
The displacement spiral is not operative in all sectors. AI can \textit{increase} demand for human labor through several channels: (a) \textit{productivity complementarity}---when AI augments rather than substitutes for human judgment, raising the marginal product of human labor (e.g., AI-assisted medical diagnosis increases demand for physicians who interpret and act on AI recommendations); (b) \textit{demand expansion}---when AI-driven cost reductions expand the market for a service beyond what was previously affordable (e.g., lower-cost legal or financial advice reaching underserved populations); (c) \textit{new categories}---when AI creates entirely new goods and services that require human labor in production, delivery, or oversight (e.g., AI safety engineering, human-AI interaction design, AI-generated content curation). Sectors where constraints are physical (construction, agriculture, healthcare delivery), regulated (legal practice, financial advisory, aviation), or trust-intensive (childcare, therapy, leadership) are likely to experience complementarity rather than substitution, at least in the near term. Our sector examples in the intermediation section are weighted toward industries where output is digitizable and switching costs are informational; this is deliberate for the stress test, but we note that the economy contains substantial sectors where AI's labor market effect may be neutral or positive.
\end{remark}

\begin{remark}[Three Modes of Displacement]
\label{rem:modes}
The labor share $s_{L,t}$ can decline through three distinct channels, each with different observational signatures:
\begin{enumerate}[label=(\alph*)]
\item \textit{Layoff substitution}: Direct headcount reduction, observable in unemployment claims and JOLTS data.
\item \textit{Wage compression}: Employment maintained but wages fall as AI provides an outside option for employers, observable in wage distribution data.
\item \textit{Hiring freeze}: New positions not created, observable in job postings data with a lag.
\end{enumerate}
Mode (c) is likely to appear first empirically but is hardest to detect in real time, as it manifests as an absence (jobs not created) rather than an event (jobs lost). This creates an ``observation gap'' in which the displacement spiral may be well underway before standard labor market indicators register distress.
\end{remark}

\subsection{Ghost GDP: The Velocity Collapse}
\label{subsec:ghostgdp}

We now formalize the divergence between measured output and consumption-relevant income.

\begin{definition}[Ghost GDP]
\label{def:ghostgdp}
Define Ghost GDP as the gap between measured real output growth and consumption-weighted income growth:
\begin{equation}
G_t^{\text{ghost}} \equiv \frac{\dot{Y}_t}{Y_t} - \frac{\dot{W}_t}{W_t},
\label{eq:ghostgdp}
\end{equation}
where $W_t = w_t L_t$ is aggregate labor income. Ghost GDP is positive when output growth exceeds labor income growth---i.e., when productivity gains accrue to capital and AI rather than to workers.
\end{definition}

The macroeconomic consequence of Ghost GDP operates through monetary velocity. In the quantity equation $MV = PY$, velocity $V = PY/M$ measures how frequently the money supply circulates through transactions. When output is produced by labor, each dollar of GDP generates approximately one dollar of labor income, which is spent with MPC $\bar{c}$, generating further transactions. When output is produced by AI, the same dollar of GDP generates capital income with lower MPC, reducing the transaction chain.

\begin{proposition}[Monetary Velocity Decline]
\label{prop:velocity}
Let $V_t$ denote monetary velocity. Under the consumption function~\eqref{eq:consumption}, velocity is:
\begin{equation}
V_t = V_0 \cdot \left[\bar{c} \cdot s_{L,t} + (1-\bar{c}) \cdot (1 - s_{L,t}) + \tau_t \right],
\label{eq:velocity}
\end{equation}
where $\tau_t \geq 0$ is the fiscal transfer rate (redistribution from capital to labor income). In the absence of policy intervention ($\tau_t = 0$):
\begin{enumerate}[label=(\roman*)]
\item $V_t$ is strictly increasing in $s_{L,t}$ (since $\bar{c} > 1 - \bar{c}$).
\item As AI adoption increases and $s_{L,t}$ declines, $V_t$ declines monotonically.
\item The rate of velocity decline (under $\tau_t = 0$) is $\dot{V}_t / V_t = (2\bar{c} - 1) \cdot \dot{s}_{L,t} / [s_{L,t}(2\bar{c}-1) + (1-\bar{c})]$.
\end{enumerate}
\end{proposition}

\begin{proof}
From~\eqref{eq:consumption}, total consumption is $C_t = [\bar{c} \cdot s_{L,t} + (1-\bar{c})(1-s_{L,t})] \cdot Y_t$. The velocity of money in a consumption-driven economy is proportional to the consumption-to-output ratio plus a transfer term: $V_t \propto C_t/Y_t + \tau_t$. Since $\partial C_t / \partial s_{L,t} = (2\bar{c} - 1) Y_t > 0$ when $\bar{c} > 1/2$ (which holds empirically, as labor income MPC substantially exceeds capital income MPC), velocity is increasing in the labor share. Parts (ii) and (iii) follow from differentiation.
\end{proof}

\begin{corollary}[Misleading Indicators]
\label{cor:misleading}
Ghost GDP implies that during a rapid AI transition, standard demand-side indicators---GDP growth, labor productivity, corporate earnings---can remain positive or even accelerate while the consumption-relevant economy contracts. Policymakers relying on these indicators will underestimate the severity of the demand shortfall.
\end{corollary}

\begin{remark}[Competing Mechanism: AI-Driven Deflation and Real Income Effects]
\label{rem:deflation}
The Ghost GDP mechanism assumes that AI-generated output growth does not translate into consumer purchasing power. A competing channel partially offsets this: if AI reduces the cost of goods and services, the resulting \textit{deflation} (or disinflation) raises real wages and real consumption even when nominal wages stagnate. Formally, let $p_t$ denote the price level and suppose AI reduces production costs at rate $g_p > 0$. Then real wages $w_t / p_t$ can be non-decreasing even if nominal wages $w_t$ are flat, provided $g_p \geq 0$. This channel is strongest when AI's cost reductions pass through to consumer prices rapidly (high competition, low markups) and when consumption is weighted toward AI-affected goods. The net effect on consumption depends on the race between \textit{nominal income compression} (the displacement spiral) and \textit{price deflation} (the real income effect). Our stress test implicitly assumes that nominal income compression dominates---an assumption that is more likely to hold when (a) AI cost reductions accrue to corporate margins rather than consumer prices (imperfect competition), (b) consumption is weighted toward services with limited AI cost reduction (housing, healthcare, education), and (c) debt service obligations are nominal (fixed mortgage payments do not adjust for deflation). Conditions (a)--(c) are plausible but not inevitable; the deflation channel provides a genuine competing mechanism that could attenuate the crisis, particularly if product market competition is vigorous.
\end{remark}

\begin{remark}[Institutional Determinants of AI Benefit Distribution]
\label{rem:institutions}
Whether AI's productivity gains accrue to workers or capital owners is not technologically determined---it depends on labor market institutions and bargaining structures. Several institutional channels could direct AI gains toward workers: (a) \textit{collective bargaining}---unions and professional associations can negotiate for AI-augmented productivity gains to be shared through higher wages, shorter hours, or improved working conditions; (b) \textit{profit-sharing and equity participation}---firms that offer employee stock ownership or profit-sharing plans automatically distribute a portion of AI-driven productivity gains to workers; (c) \textit{tight labor markets}---when unemployment is low and labor supply is inelastic, competitive pressure forces firms to share productivity gains through wages to retain workers; (d) \textit{minimum wage and social floors}---binding wage floors compress the distribution of wages from below, ensuring that at least some AI gains reach low-wage workers. Our model's consumption function~\eqref{eq:consumption} parameterizes this institutional structure through the labor share $s_{L,t}$, which declines only when the institutional mechanisms connecting productivity to wages weaken faster than AI capability grows. Historically, the U.S.\ institutional environment has been relatively unfavorable for labor capture of productivity gains since the 1970s (Figure~\ref{fig:macro}(d)); economies with stronger collective bargaining (e.g., Scandinavia, Germany) may experience a different displacement trajectory.
\end{remark}

\subsection{Intermediation Collapse}
\label{subsec:intermediation}

A substantial share of service-sector value added derives from intermediation: firms that extract margin by reducing information frictions between buyers and sellers. We model this explicitly.

Let $\phi$ denote the level of information friction in a market---search costs, comparison costs, switching costs, complexity rents. An intermediary extracts margin $m$ that is increasing in friction:
\begin{equation}
m(\phi) = m_0 + \gamma_m \cdot \phi,
\label{eq:margin}
\end{equation}
where $m_0 \geq 0$ is the margin from pure logistics/infrastructure (non-friction value) and $\gamma_m > 0$ is the friction-margin sensitivity.

AI agents reduce friction by performing search, comparison, and negotiation tasks that previously required human effort or human inattention:
\begin{equation}
\phi(A_t) = \phi_0 \cdot e^{-\gamma_\phi A_t},
\label{eq:friction}
\end{equation}
where $\phi_0$ is the pre-AI friction level and $\gamma_\phi > 0$ is the AI-friction reduction rate.

\begin{proposition}[Intermediation Margin Compression]
\label{prop:intermediation}
As AI capability $A_t$ rises:
\begin{enumerate}[label=(\roman*)]
\item Intermediary margins converge to the infrastructure floor: $\lim_{A_t \to \infty} m(\phi(A_t)) = m_0$.
\item The rate of margin compression is $\dot{m}_t = -\gamma_m \gamma_\phi \phi_0 e^{-\gamma_\phi A_t} \cdot g_A A_t$.
\item Revenue at risk is $\Delta R_t = \gamma_m \cdot [\phi_0 - \phi(A_t)] \cdot Q_t$, where $Q_t$ is transaction volume.
\end{enumerate}
\end{proposition}

\begin{proof}
Direct substitution of~\eqref{eq:friction} into~\eqref{eq:margin} yields $m_t = m_0 + \gamma_m \phi_0 e^{-\gamma_\phi A_t}$. Part (i) follows from $\lim_{A_t \to \infty} e^{-\gamma_\phi A_t} = 0$. Part (ii) follows from differentiation using $\dot{A}_t = g_A A_t$. Part (iii) follows from the difference $m_0 + \gamma_m \phi_0 - m_t = \gamma_m[\phi_0 - \phi(A_t)]$ applied to transaction volume $Q_t$.
\end{proof}

\begin{remark}[Margin Compression versus Industry Elimination]
\label{rem:compression}
Proposition~\ref{prop:intermediation} predicts margin compression, not industry disappearance. Intermediaries that provide genuine infrastructure value ($m_0 > 0$) survive but at lower margins. The empirical question is the ratio $m_0 / (m_0 + \gamma_m \phi_0)$---the share of current margin attributable to non-friction value. Industries with high friction-to-value ratios face the largest repricing.
\end{remark}

\begin{remark}[Countervailing Moats and Friction Floors]
\label{rem:moats}
The frictionless limit $\phi \to 0$ is an asymptotic benchmark, not a realistic endpoint. Several forces create friction floors that limit the rate and extent of intermediation margin compression:
\begin{enumerate}[label=(\alph*)]
\item \textit{Regulatory compliance}: financial services, healthcare, and legal intermediation embed compliance costs that AI cannot eliminate and may increase (e.g., KYC/AML for payments, fiduciary duties for advisory). These create a binding floor $\phi_{\min}^{\text{reg}} > 0$.
\item \textit{Trust and liability}: for high-stakes decisions (insurance underwriting, investment advice, medical referrals), consumers and regulators require accountable human intermediaries. AI may augment rather than replace these roles, limiting margin compression to the informational component.
\item \textit{Switching costs and lock-in}: network effects, data portability barriers, and contractual lock-in (e.g., multi-year enterprise SaaS contracts) slow the speed at which friction reduction translates into margin compression.
\item \textit{Brand and distribution}: physical distribution networks, brand trust, and established customer relationships constitute non-informational moats that AI agents cannot easily arbitrage away.
\end{enumerate}
These moats imply that the effective friction function is $\phi(A_t) = \max\{\phi_{\min}, \; \phi_0 \cdot e^{-\gamma_\phi A_t}\}$, where $\phi_{\min} > 0$ reflects the regulatory and institutional friction floor. The resulting margin floor $m_0 + \gamma_m \phi_{\min}$ may be substantially above $m_0$, particularly in regulated sectors. Our stress test explores the implications of rapid movement toward this floor, not the elimination of all intermediation value.
\end{remark}

\paragraph{Applications.} The intermediation collapse mechanism applies with varying intensity across sectors, depending on the friction-to-moat ratio. Table~\ref{tab:intermediation} quantifies the sector-specific exposure.

\begin{table}[!htbp]
\centering
\caption{Intermediation Exposure by Sector}
\label{tab:intermediation}
\footnotesize
\begin{tabular}{@{}p{2cm}p{1.3cm}p{1.6cm}p{1.6cm}p{1.8cm}p{1.8cm}@{}}
\toprule
\textbf{Sector} & \textbf{U.S.\ Rev. (\$B)} & \textbf{Friction Share} & \textbf{Switching Costs} & \textbf{Regulatory Barrier} & \textbf{Net Exposure} \\
\midrule
SaaS (seat) & $\sim$300 & High (60--80\%) & Moderate & Low & \textbf{High} \\
\addlinespace
Card payments & $\sim$120 & Moderate (40--60\%) & High & High (Durbin, PCI) & Moderate \\
\addlinespace
Insurance brokerage & $\sim$50 & Moderate (40--50\%) & Moderate & High (licensing) & Low--Mod. \\
\addlinespace
Mgmt.\ consulting & $\sim$330 & High (50--70\%) & Low & Low & \textbf{High} \\
\addlinespace
Financial advisory & $\sim$120 & Moderate (30--50\%) & Moderate & High (RIA, fiduciary) & Moderate \\
\addlinespace
Legal services & $\sim$370 & Moderate (30--50\%) & Moderate & High (bar, UPL) & Low--Mod. \\
\addlinespace
Travel booking & $\sim$60 & High (60--80\%) & Low & Low & \textbf{High} \\
\bottomrule
\multicolumn{6}{@{}p{12cm}@{}}{\scriptsize \textit{Notes:} Revenue estimates are approximate annual U.S.\ figures. ``Friction share'' estimates the fraction of current margin attributable to information frictions versus non-friction value. ``Net exposure'' reflects friction share, switching costs, and regulatory barriers combined.} \\
\end{tabular}
\end{table}

The table reveals substantial heterogeneity: SaaS (seat-based pricing), management consulting, and travel booking are highly exposed because their margins derive primarily from information frictions with low regulatory barriers. By contrast, financial advisory and legal services have substantial regulatory moats (licensing, fiduciary requirements, unauthorized practice rules) that create friction floors independent of AI capability. Card-based payments occupy a middle ground: interchange fees embed a large friction component, but network effects, consumer protection regulations, and fraud liability frameworks create substantial switching costs. Stablecoin settlement, while growing (estimated \$390 billion in 2025), remains small relative to global card transaction volume (\$10+ trillion annually) and lacks equivalent consumer protection infrastructure (chargeback mechanisms, fraud liability limits), suggesting that payments displacement will proceed gradually and is contingent on regulatory evolution \citep{mckinsey2025stablecoins}. The break-even economics are instructive: card interchange of 2--3\% funds consumer rewards ($\sim$1\% cashback), fraud protection ($\sim$0.3\% of volume), and issuer margin; stablecoin rails eliminate interchange but impose gas fees, fiat on/off-ramp costs (0.5--1.5\%), and---critically---shift fraud and chargeback liability entirely to merchants or consumers. Until regulatory frameworks allocate these costs comparably across rails, the effective price advantage of stablecoin settlement is substantially smaller than the headline interchange saving, particularly for consumer-facing transactions.

\subsection{Macro-Financial Transmission}
\label{subsec:transmission}

We now connect the three mechanisms to financial markets through the consumption concentration channel.

\begin{proposition}[Consumption Concentration Amplifier]
\label{prop:concentration}
Partition households into income quintiles $q \in \{1, \ldots, 5\}$ with income $Y_q$, consumption $C_q = c_q Y_q$ (where $c_q$ is the quintile-specific MPC), and AI-displacement exposure $\delta_q$. The aggregate consumption shock from AI displacement is:
\begin{equation}
\Delta C = -\sum_{q=1}^{5} c_q \cdot \delta_q \cdot Y_q.
\label{eq:concshock}
\end{equation}
If the top quintile ($q = 5$) accounts for share $\chi$ of aggregate consumption and has displacement exposure $\delta_5$, then:
\begin{equation}
\frac{|\Delta C|}{C} \geq \chi \cdot \delta_5,
\label{eq:amplifier}
\end{equation}
which exceeds the labor-market share of the top quintile in total employment.
\end{proposition}

\begin{proof}
The aggregate consumption shock~\eqref{eq:concshock} weights each quintile's displacement by its consumption share $c_q Y_q / C$. By definition, $\chi = c_5 Y_5 / C$, so the top-quintile contribution is $c_5 \delta_5 Y_5 / C = \chi \delta_5$. Since $\delta_q \geq 0$ for all $q$, the total shock weakly exceeds the top-quintile component, yielding~\eqref{eq:amplifier}.
\end{proof}

Empirically, estimates of $\chi$ range from 0.47 to 0.65 depending on the survey instrument: Consumer Expenditure Survey (CES) methods yield approximately 47--50\%, while approaches reconciled with the Survey of Consumer Finances (SCF) and national accounts yield approximately 59--65\% \citep{moodys2023consumption}. We use $\chi = 0.59$ as a central estimate, following \citet{moodys2023consumption}, but conduct sensitivity analysis across the full range. A 10\% income reduction for the top quintile ($\delta_5 = 0.10$) produces an aggregate consumption decline of 4.7--6.5\% depending on $\chi$---spanning the range of consumption collapse observed during the 2008--2009 Global Financial Crisis.

\begin{proposition}[Credit Transmission]
\label{prop:credit}
Consider borrowers with income $Y_i$ and debt service obligation $D_i$. Define the debt service coverage ratio $r_i = Y_i / D_i$. The default probability is:
\begin{equation}
P_D(r_i) = \Phi\left(\frac{\ln(1) - \ln(r_i)}{\sigma_r}\right),
\label{eq:default}
\end{equation}
where $\Phi(\cdot)$ is the standard normal CDF and $\sigma_r$ is income volatility. Then:
\begin{enumerate}[label=(\roman*)]
\item $P_D$ is convex in income reduction $\delta$ for borrowers near the threshold ($r_i$ close to 1): $\partial^2 P_D / \partial \delta^2 > 0$.
\item A structural increase in income volatility $\sigma_r$ (from AI-driven job instability) shifts the entire default probability distribution upward.
\item For borrowers with $r_i > 1$ (currently performing), the default probability is increasing and convex in the magnitude of the income shock.
\end{enumerate}
\end{proposition}

\begin{proof}
From~\eqref{eq:default}, an income reduction $\delta$ reduces $r_i$ to $r_i(1-\delta)$, so $P_D(\delta) = \Phi((-\ln r_i + \ln(1-\delta)^{-1})/\sigma_r)$. The first derivative with respect to $\delta$ is $\phi(\cdot) / [\sigma_r(1-\delta)] > 0$, and the second derivative includes a term $\phi(\cdot)/[\sigma_r(1-\delta)^2] > 0$ (where $\phi$ is the standard normal PDF), establishing convexity. Parts (ii) and (iii) follow from the monotone likelihood ratio property of the normal distribution.
\end{proof}

The convexity result in Proposition~\ref{prop:credit} is critical: in a standard recession, income losses are temporary and distributed across quintiles, so default risk increases approximately linearly. Under AI displacement, income losses are potentially permanent and concentrated in high-income cohorts with large debt obligations, producing a \textit{convex} increase in default probability.

\paragraph{Quantifying credit exposure: a DSCR sensitivity exercise.} To provide an empirical bridge between the theoretical default model and credit market implications, consider a representative prime borrower with initial debt-service coverage ratio $r_0 = 1.5$ (income 50\% above debt service obligations) and income volatility $\sigma_r = 0.20$. Under these parameters, the baseline annual default probability is $P_D \approx \Phi(-\ln(1.5)/0.20) \approx \Phi(-2.03) \approx 2.1\%$. A 20\% permanent income reduction ($\delta = 0.20$) reduces $r$ to 1.2 and raises $P_D$ to $\Phi(-\ln(1.2)/0.20) \approx \Phi(-0.91) \approx 18.1\%$---a nearly ninefold increase. A 30\% reduction ($\delta = 0.30$) raises $P_D$ to $\Phi(-\ln(1.05)/0.20) \approx \Phi(-0.24) \approx 40.3\%$. The convexity is apparent: the second 10 percentage points of income loss raise default probability by more than the first.

For this transmission to produce systemic stress, three additional conditions must hold: (i) \textit{persistence}---income losses must be structural rather than cyclical, so that borrowers cannot ``wait out'' the shock; (ii) \textit{correlation}---income losses must be correlated across borrowers in the same cohort (e.g., white-collar workers in AI-exposed industries in the same metropolitan area); and (iii) \textit{inability to restructure}---refinancing or loan modification must be constrained (e.g., by rising rates, negative equity, or lender unwillingness). AI displacement satisfies condition (i) if the technological substitution is permanent, and condition (ii) if displacement is concentrated by occupation and geography. Condition (iii) depends on the interest rate environment and housing market conditions, which are independent of AI. We note these conditions explicitly because the leap from ``borrower-level default risk increases'' to ``systemic mortgage crisis'' requires all three to hold simultaneously.

\paragraph{Private credit exposure.} Global private credit assets under management exceeded \$2.5 trillion by 2024 \citep{bis2024privatecredit, imf2024gfsr}. A meaningful share has been deployed into leveraged buyouts of software and technology services companies at valuations assuming sustained revenue growth. The intermediation collapse mechanism (Proposition~\ref{prop:intermediation}) directly threatens these revenue assumptions. However, we emphasize that private credit fragility is substantially a \textit{pre-existing condition}: leverage, covenant erosion, valuation opacity, and the PE-insurer linkage are concerns that the BIS and FSB have flagged independently of AI \citep{bis2024privatecredit, fsb2023nbfi}. AI-driven revenue compression would be one potential trigger for repricing, but rising interest rates, sectoral downturns, or regulatory changes could produce similar effects. The distinctive contribution of AI displacement is not that it creates financial fragility \textit{de novo}, but that it provides a correlated shock mechanism across otherwise diversified private credit portfolios. A key early-warning indicator is the gap between private marks and public-market comparables: private credit funds typically report net asset values with a 60--90 day lag and use model-based rather than market-based valuations; during stress episodes, public credit spreads widen months before private marks adjust, creating a ``valuation gap'' that understates realized losses and delays capital calls from insurer limited partners \citep{bis2024privatecredit}.

\paragraph{Mortgage market exposure.} The U.S.\ residential mortgage market totals approximately \$13 trillion \citep{nyfed2025mortgage}. AI displacement presents a novel risk: borrowers with prime credit histories (780+ FICO scores, 20\% down payments) whose income stability deteriorates \textit{after} origination due to structural technological change. The loans were sound at origination; the assumption of income stability that underlay them is what changes. Regional concentration is relevant: if AI displacement is geographically concentrated in technology and finance hubs (San Francisco, Seattle, Austin, New York), mortgage stress would first appear in these metros---potentially triggering localized housing corrections before any systemic effect. Historical precedent suggests that geographically concentrated income shocks (e.g., Detroit's automotive manufacturing collapse in the 2000s, which saw house prices fall 40\% and foreclosure rates triple; Houston's energy bust in the 1980s, which produced a 30\% housing price decline and widespread savings-and-loan failures) can produce severe local housing distress without necessarily cascading into a national crisis; the systemic risk depends on the geographic breadth of AI displacement and whether securitization has distributed the exposure nationally, as mortgage-backed securities did for subprime exposure in 2007--2008.

\section{Empirical Analysis}
\label{sec:empirics}

We ground the model in three layers of empirical evidence: (i) macro time series documenting the secular preconditions for the mechanisms we formalize, (ii) occupation-level cross-sectional regressions linking AI exposure to wage outcomes, and (iii) a calibrated simulation that disciplines the displacement dynamics with empirical moments.

\subsection{Macro Preconditions: Labor Share and Velocity}
\label{subsec:macroevidence}

The displacement spiral and Ghost GDP mechanisms operate through the labor share of income and monetary velocity. We document their secular trends using FRED data \citep{fred2025data}. Figure~\ref{fig:macro} presents four macro preconditions.

\begin{figure}[H]
\centering
\includegraphics[width=\textwidth]{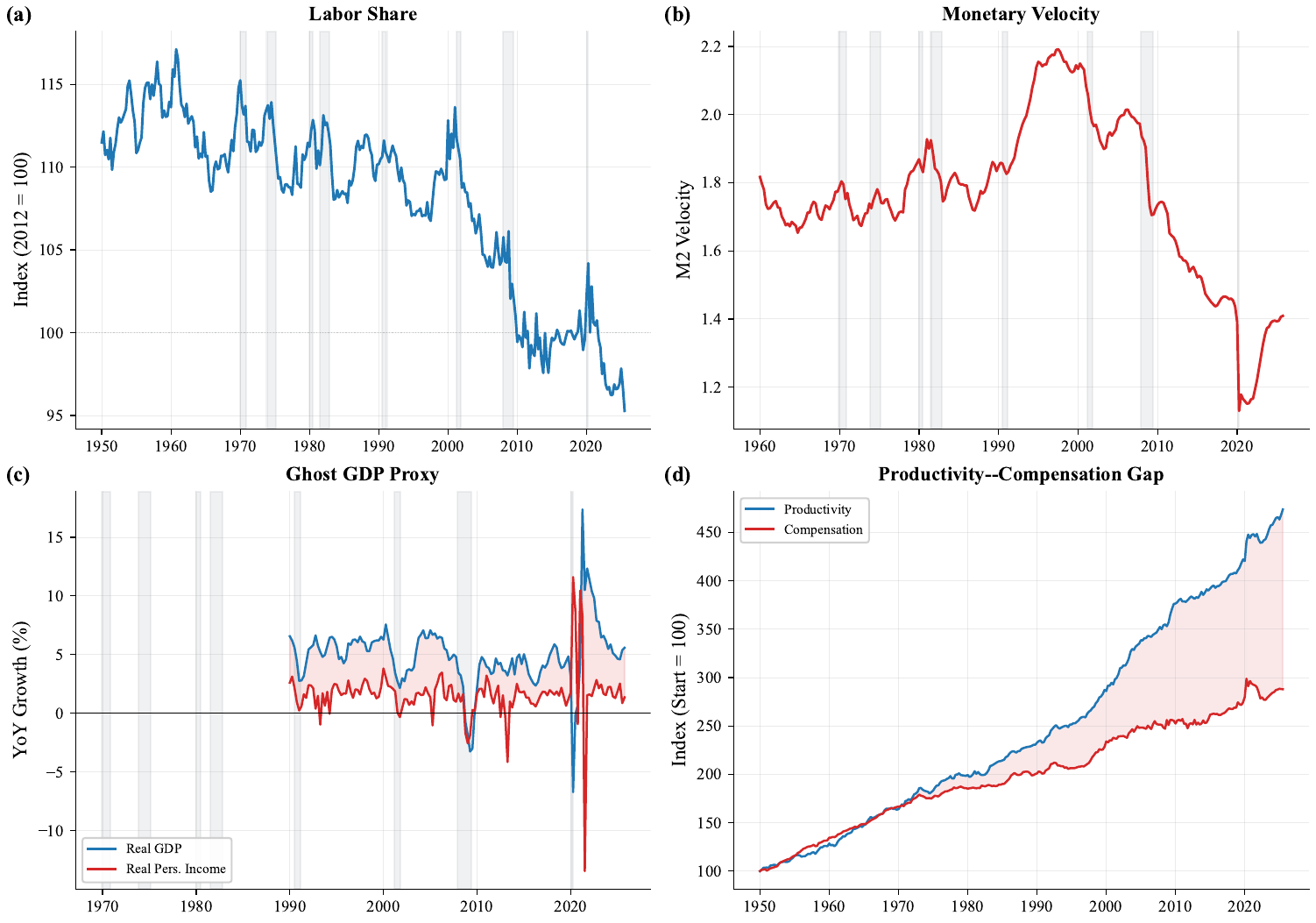}
\caption{Macro preconditions for the displacement spiral. (a)~Labor share (BLS, FRED series PRS85006173) has declined $\sim$19\% from its 1960 peak. (b)~M2 velocity (FRED series M2V) has declined 36\% from its 1997 peak, consistent with Proposition~\ref{prop:velocity}. (c)~Ghost GDP proxy: divergence between real GDP and real personal income growth, widening post-2020. (d)~Productivity--compensation gap since the 1970s demonstrates that technological progress need not translate into wage growth. Shaded regions indicate NBER recession dates.}
\label{fig:macro}
\end{figure}

The labor share has declined approximately 19\% over six decades, driven first by globalization and automation and later by superstar firms \citep{karabarbounis2014global, autor2020fall}. M2 velocity declined from 2.19 (1997) to 1.41 (2025)---consistent with income shifting from high-MPC labor to low-MPC capital. The Ghost GDP proxy (Panel~c) shows persistent post-2020 divergence between output and personal income growth. The productivity--compensation gap (Panel~d), open since the 1970s, is the structural precondition: AI threatens to accelerate this divergence.

\subsection{Occupation-Level Evidence: AI Exposure and Wage Growth}
\label{subsec:occevidence}

We next examine whether AI exposure is already associated with differential wage outcomes at the occupation level. We construct a panel of 22 two-digit SOC occupation groups using BLS Occupational Employment and Wage Statistics for 2019 (pre-ChatGPT) and 2023 (post-ChatGPT), and assign each occupation group an AI exposure score based on the composite indices of \citet{felten2021occupational} and \citet{eloundou2023gpts}.

\begin{figure}[H]
\centering
\includegraphics[width=\textwidth]{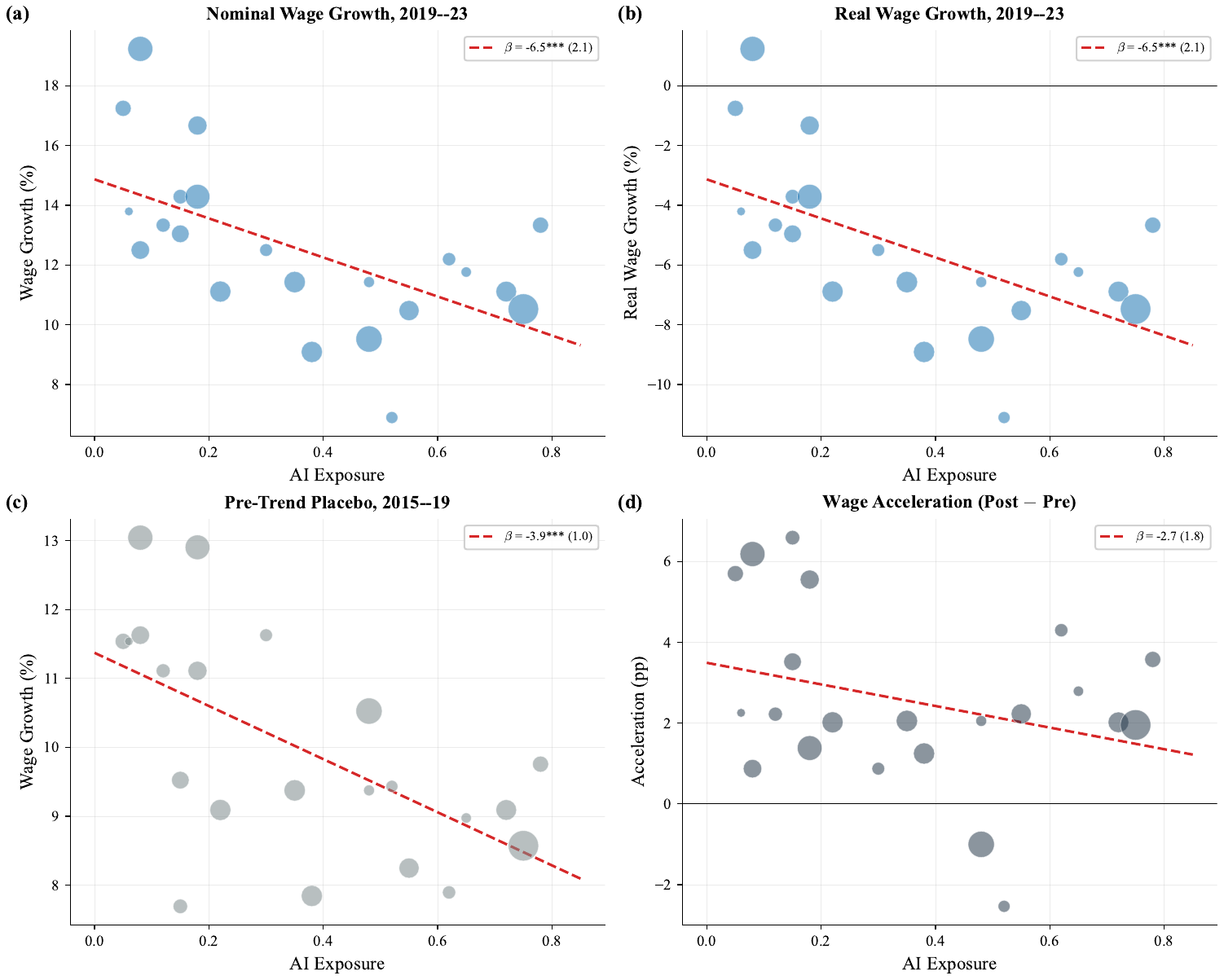}
\caption{Occupation-level evidence: AI exposure and wage outcomes. (a)~Nominal wage growth (2019--2023) vs.\ AI exposure ($\beta = -6.53$, $p < 0.01$). (b)~Real wage growth after adjusting for 18\% CPI inflation. (c)~Pre-trend placebo (2015--2019): the significant negative slope ($\beta = -3.86$, $p < 0.01$) indicates pre-existing trends, likely reflecting the denominator effect (higher-wage occupations have lower percentage growth). (d)~Wage acceleration (DiD): the post-minus-pre difference ($\beta = -2.68$, $p = 0.14$) isolates the AI-specific shift. Bubble size proportional to 2019 employment. Source: BLS OEWS; AI exposure based on \citet{felten2021occupational} and \citet{eloundou2023gpts}.}
\label{fig:occupation}
\end{figure}

Figure~\ref{fig:occupation} visualizes the occupation-level relationships; Table~\ref{tab:regressions} presents formal regression results.

\begin{table}[H]
\centering
\caption{Cross-Sectional Regressions: Occupation-Level Outcomes on AI Exposure}
\label{tab:regressions}
\small
\begin{tabular}{@{}lcccccc@{}}
\toprule
 & (1) & (2) & (3) & (4) & (5) & (6) \\
 & Nom.\ Wage & Real Wage & Employment & Nom.\ Wage & Pre-Trend & Wage \\
 & Growth & Growth & Growth & Growth & Placebo & Accel.\ \\
 & 2019--23 & 2019--23 & 2019--23 & 2019--23 & 2015--19 & $\Delta\Delta$ \\
\midrule
AI Exposure & -6.53*** & -6.53*** & 2.22 & -4.21** & -3.86*** & -2.68 \\
 & (2.11) & (2.11) & (4.16) & (2.07) & (1.03) & (1.79) \\
\addlinespace
Log Wage 2019 & & & & -1.81 & & \\
 & & & & (1.29) & & \\
\addlinespace
Constant & 14.87 & -3.13 & 3.27 & 33.54 & 11.37 & 3.49 \\
 & (0.92) & (0.92) & (1.41) & (13.92) & (0.55) & (0.82) \\
\midrule
$R^2$ & 0.337 & 0.337 & 0.022 & 0.370 & 0.351 & 0.089 \\
$N$ & 22 & 22 & 22 & 22 & 22 & 22 \\
\bottomrule
\multicolumn{7}{@{}p{15cm}@{}}{\scriptsize \textit{Notes:} OLS regressions at the 2-digit SOC occupation group level. Columns (1)--(4): dependent variables are percentage changes between 2019 and 2023. Column (5): pre-trend placebo---dependent variable is 2015--2019 nominal wage growth, predating generative AI. The significant negative coefficient ($p < 0.01$) indicates pre-existing differential wage trends in AI-exposed occupations, likely reflecting the denominator effect (higher-wage occupations have lower percentage growth); see Section~\ref{subsec:occevidence} for discussion. Column (6): wage acceleration (difference-in-differences style)---dependent variable is the change in wage growth rate, $({\Delta w_{2019-23}} - {\Delta w_{2015-19}})$, isolating the post-AI shift after netting out pre-existing trends. AI Exposure is a composite score based on \citet{felten2021occupational} and \citet{eloundou2023gpts}. Robust standard errors (HC1) in parentheses. * $p<0.10$, ** $p<0.05$, *** $p<0.01$.} \\
\end{tabular}
\end{table}

Columns (1)--(2) show that a one-unit increase in AI exposure is associated with 6.53 pp lower wage growth ($p < 0.01$), explaining 34\% of cross-sectional variation. Column (3) shows that employment growth is not significantly related to AI exposure ($p > 0.10$), consistent with displacement operating initially through wage compression rather than layoffs. Column (4) controls for log initial wage; the coefficient remains significant ($\beta = -4.21$, $p < 0.05$).

\paragraph{Causal identification.} The OLS results are suggestive but not causal. The pre-trend placebo (Column 5) reveals that AI-exposed occupations already had lower percentage wage growth in 2015--2019 ($\beta = -3.86$, $p < 0.01$), likely reflecting the denominator effect and pre-existing automation pressure. The wage acceleration test (Column 6) nets out the pre-trend: the coefficient is negative ($\beta = -2.68$) but imprecisely estimated ($p = 0.14$). The total effect decomposes into pre-trend ($-3.86$ pp) and post-AI acceleration ($-2.68$ pp). An ideal identification strategy would exploit exogenous variation in AI tool deployment across comparable occupations; we note this as a priority for future work. Online Appendix~C reports robustness checks using alternative exposure measures, weighting, and nonlinear specifications.

Unlike wage growth, the relationship between AI exposure and employment growth is not statistically significant (Online Appendix~A, Figure~A1), consistent with displacement operating initially through wage compression and hiring freezes (Remark~\ref{rem:modes}(b)--(c)) rather than direct layoffs.

\subsection{Calibration and Simulation}
\label{subsec:calibration}

We calibrate the key model parameters using the empirical moments documented above and external sources. Table~\ref{tab:calibration} presents the calibrated values and their sources.

\begin{table}[!htbp]
\centering
\caption{Calibrated Model Parameters}
\label{tab:calibration}
\renewcommand{\arraystretch}{1.2}
\begin{tabular}{@{}p{4.8cm}lcl@{}}
\toprule
\textbf{Parameter} & \textbf{Symbol} & \textbf{Value} & \textbf{Source} \\
\midrule
\multicolumn{4}{@{}l}{\textit{Labor Market}} \\[2pt]
Initial labor share & $s_{L,0}$ & 0.56 & BLS (2024) \\
MPC, labor income & $\bar{c}$ & 0.85 & Carroll et al.\ [2017] \\
MPC, capital income & $1-\bar{c}$ & 0.15 & Model restriction$^\dagger$ \\
\addlinespace[4pt]
\multicolumn{4}{@{}l}{\textit{Consumption}} \\[2pt]
Top-quintile consump.\ share & $\chi$ & 0.47--0.65 & Moody's [2023]; CES/SCF \\
PCE / GDP & & 0.68 & BEA / FRED \\
\addlinespace[4pt]
\multicolumn{4}{@{}l}{\textit{AI Dynamics}} \\[2pt]
AI capability growth (base) & $g_A$ & 0.05 & Conservative estimate \\
AI capability growth (rapid) & $g_A$ & 0.20 & METR [2025a] \\
AI cost decline rate & $g_c$ & 0.30 & Industry estimates \\
Diffusion ceiling & $\bar{d}$ & 0.80 & Zolas et al.\ [2024] \\
Diffusion speed & $\kappa$ & 2.0 & Enterprise adoption lag \\
Reinstatement (baseline) & $\rho_0$ & 0.002 & Acemoglu \& Restrepo [2019] \\
Complementarity & $\eta$ & 0.003 & Acemoglu \& Restrepo [2019] \\
Reinstatement concavity & $\alpha_\rho$ & 0.50 & Calibrated \\
Displacement feedback & $\beta$ & 0.30 & Calibrated \\
Direct subst.\ sensitivity & $f'(g_A)$ & 0.15 & Acemoglu [2024] \\
\addlinespace[4pt]
\multicolumn{4}{@{}l}{\textit{Financial Markets}} \\[2pt]
M2 velocity (2025) & $V_0$ & 1.41 & FRED (M2V) \\
U.S.\ mortgage market & & \$13.17T & NY Fed [2025] \\
Global private credit AUM & & \$2.5T & BIS [2024] \\
\bottomrule
\multicolumn{4}{@{}p{13cm}@{}}{\scriptsize $^\dagger$The consumption function (Equation~5 in the main text) restricts the two MPCs to sum to 1. The restricted value $1-\bar{c} = 0.15$ strengthens the demand feedback, making this a stress-test-conservative assumption.} \\
\end{tabular}
\end{table}

The initial labor share $s_{L,0} = 0.56$ matches the BLS nonfarm business sector labor share as of 2024. The CES elasticity of substitution is set at $\sigma = 1.0$ (Cobb-Douglas benchmark), with sensitivity analysis across $\sigma \in [0.5, 2.0]$ reported in Online Appendix~E; empirical estimates for automation technologies range from 0.8 to 1.5 \citep{acemoglu2022tasks}, and if generative AI is a closer substitute for cognitive tasks than prior automation technologies, the effective $\sigma$ may be at the higher end. The MPC out of labor income ($\bar{c} = 0.85$) follows \citet{carroll2017distribution}; the functional form of~\eqref{eq:consumption} restricts the MPC out of capital income to $1 - \bar{c} = 0.15$, below empirical estimates of 0.30--0.40 \citep{jappelli2014fiscal}. This restriction widens the MPC gap $(2\bar{c}-1)$ that drives the demand feedback, making it a stress-test-conservative assumption; using separate MPC parameters ($c_L = 0.85$, $c_K = 0.35$) would attenuate the feedback coefficient and raise the threshold $g_A^*(\rho)$ for explosive displacement. The top-quintile consumption share is set at $\chi = 0.59$ \citep{moodys2023consumption}, with sensitivity analysis across the range $\chi \in [0.47, 0.65]$ reflecting measurement differences between CES-based and SCF-reconciled approaches (Online Appendix~D). The AI capability growth rate is parameterized across three scenarios: baseline ($g_A = 0.05$), rapid ($g_A = 0.20$), and extreme ($g_A = 0.40$), bracketing the range of capability improvement rates documented by \citet{metr2025benchmarks}. The diffusion ceiling is set at $\bar{d} = 0.80$ with speed $\kappa = 2.0$, reflecting that full adoption of technically feasible automation is never instantaneous and is constrained by enterprise integration costs, verification requirements, and infrastructure readiness (including on-device inference limitations for consumer-facing agents). The reinstatement rate is calibrated to a baseline $\rho_0 = 0.002$ with complementarity $\eta = 0.003$ and concavity $\alpha_\rho = 0.5$, so that reinstatement approximately balances displacement at baseline AI adoption rates ($g_A = 0.05$)---consistent with the historical pattern documented by \citet{acemoglu2019automation}---but is overwhelmed by displacement at rapid rates ($g_A \geq 0.20$). The feedback coefficient ($\beta = 0.30$) is calibrated to match the speed of labor share decline during historical episodes of rapid automation. Online Appendix~D reports sensitivity analysis showing that the AI growth rate $g_A$ and the reinstatement rate $\rho$ are the dominant drivers of crisis depth; a Monte Carlo simulation with 2{,}000 draws sampling all parameters jointly---including reinstatement---shows that under the median draw, reinstatement absorbs displacement (the benign outcome), but the distribution has a fat tail: approximately 14\% of draws produce a crisis-level demand shortfall exceeding 30\%, comparable to the ex-ante probability of a major financial crisis in any given decade.

\paragraph{Relationship to Acemoglu's aggregate estimates.} Our calibration uses the direct substitution sensitivity parameter $f'(g_A) = 0.15$ from \citet{acemoglu2024simple}, who estimates that AI's aggregate effect on GDP is modest: approximately 0.7--1.6\% over a decade under his baseline calibration. Our stress test can produce substantially larger effects (labor share declines of 20--40\% under the rapid scenario). The discrepancy arises from three modeling differences, each reflecting a deliberate stress-test assumption:
\begin{enumerate}[nosep]
\item \textit{Feedback amplification}: Acemoglu's framework treats AI adoption as a one-directional shock to task allocation; our model adds a demand-side feedback loop ($\beta > 0$) in which displacement reduces demand, which generates margin pressure, which accelerates further displacement. This amplifier is the central mechanism of our stress test and is absent from Acemoglu's partial-equilibrium analysis.
\item \textit{Adoption speed}: Acemoglu's analysis assumes gradual adoption consistent with historical technology diffusion. Our rapid scenario ($g_A = 0.20$) assumes capability growth substantially faster than historical precedent---a stress-test assumption, not a prediction.
\item \textit{Consumption concentration}: Acemoglu's framework uses representative-agent demand. Our model adds the consumption concentration amplifier (Proposition~\ref{prop:concentration}), which transmits AI displacement concentrated among high-income workers into a disproportionate demand shock.
\end{enumerate}
If the feedback coefficient is small ($\beta \to 0$), adoption is gradual ($g_A \leq 0.05$), and consumption concentration is ignored ($\chi = 0$), our model converges to Acemoglu's modest-effect result. The crisis pathway requires all three amplifiers to be active simultaneously. Our Monte Carlo results confirm this: the median draw produces a benign outcome consistent with Acemoglu's estimates; the crisis emerges only in the tail where adoption is rapid and reinstatement is weak.

Using these parameters, we simulate the displacement dynamics of Proposition~\ref{prop:spiral} under each scenario with no policy intervention ($\tau = 0$).

\begin{figure}[H]
\centering
\includegraphics[width=\textwidth]{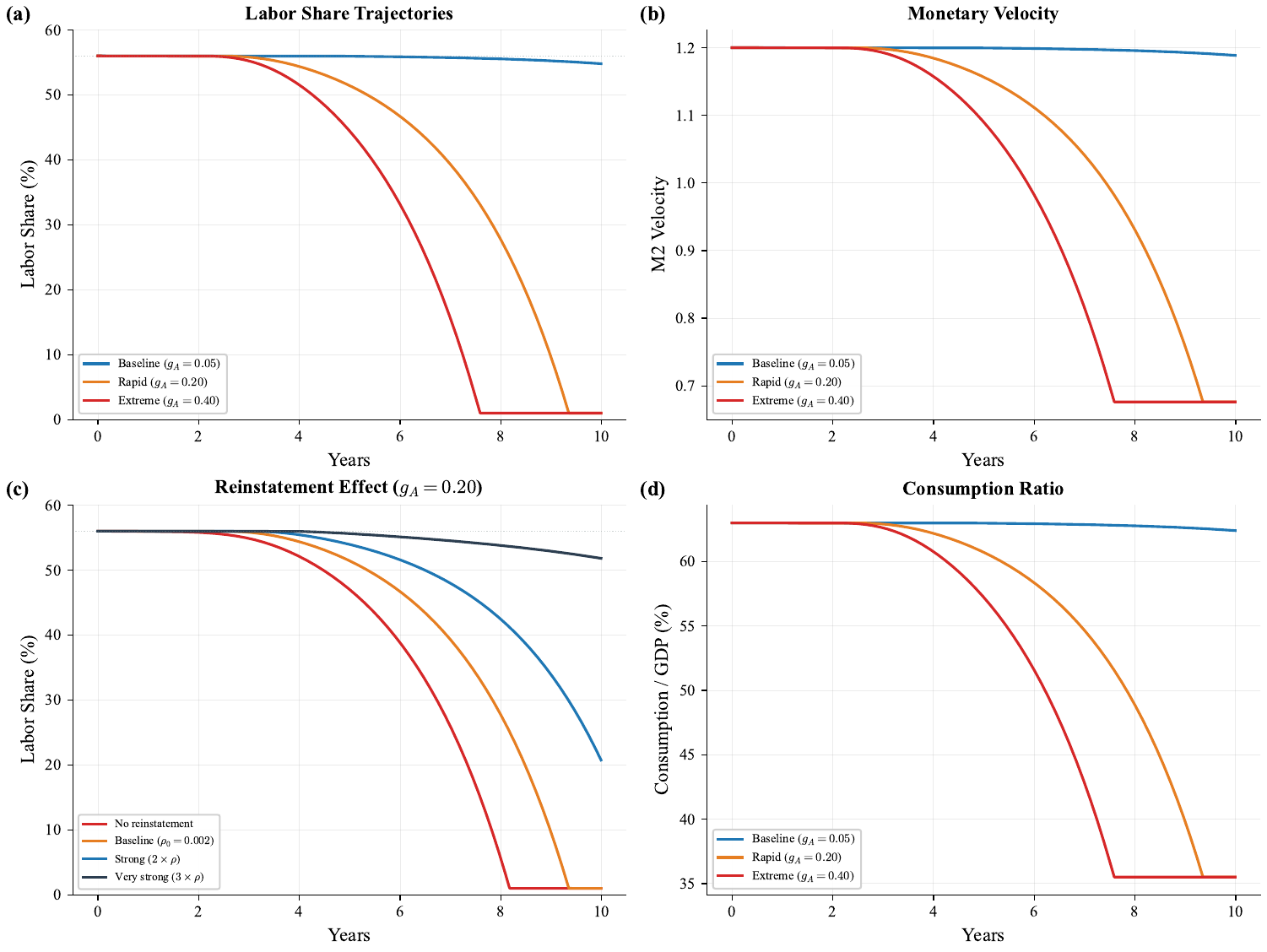}
\caption{Calibrated displacement dynamics with diffusion and reinstatement ($\rho_0 = 0.002$, $\bar{d} = 0.80$, $\kappa = 2.0$, $\tau = 0$). (a) Labor share under three AI adoption rates. (b) Monetary velocity. (c) Reinstatement comparison under the rapid scenario. (d) Consumption-to-GDP ratio.}
\label{fig:simulation}
\end{figure}

Figure~\ref{fig:simulation} displays the simulation results incorporating logistic diffusion and task reinstatement. Four features merit emphasis. First, the baseline scenario ($g_A = 0.05$) produces a gradual labor share decline that mirrors historical trends---a slow, manageable transition where reinstatement partially offsets displacement. Second, the rapid scenario ($g_A = 0.20$) produces a qualitatively different trajectory: the feedback mechanism accelerates the decline, and despite reinstatement, the labor share falls substantially within a decade. This is the ``acute'' scenario in which Ghost GDP and the consumption concentration amplifier generate crisis-level demand shortfall. Third, the extreme scenario ($g_A = 0.40$) illustrates the explosive regime: displacement overwhelms reinstatement, the labor share collapses, velocity falls sharply, and the consumption ratio contracts to levels well below those observed in any post-war recession. Fourth, Panel (c) demonstrates the reinstatement channel's quantitative importance: under the rapid scenario, tripling the reinstatement rate substantially attenuates the labor share decline, consistent with Proposition~\ref{prop:spiral}---the crisis outcome is not predetermined but depends on the balance between displacement and new task creation.

\paragraph{Policy response simulations.} Figure~\ref{fig:policy_sim} examines how policy response speed and magnitude alter these trajectories, holding the AI adoption rate at the rapid scenario ($g_A = 0.20$).

\begin{figure}[H]
\centering
\includegraphics[width=\textwidth]{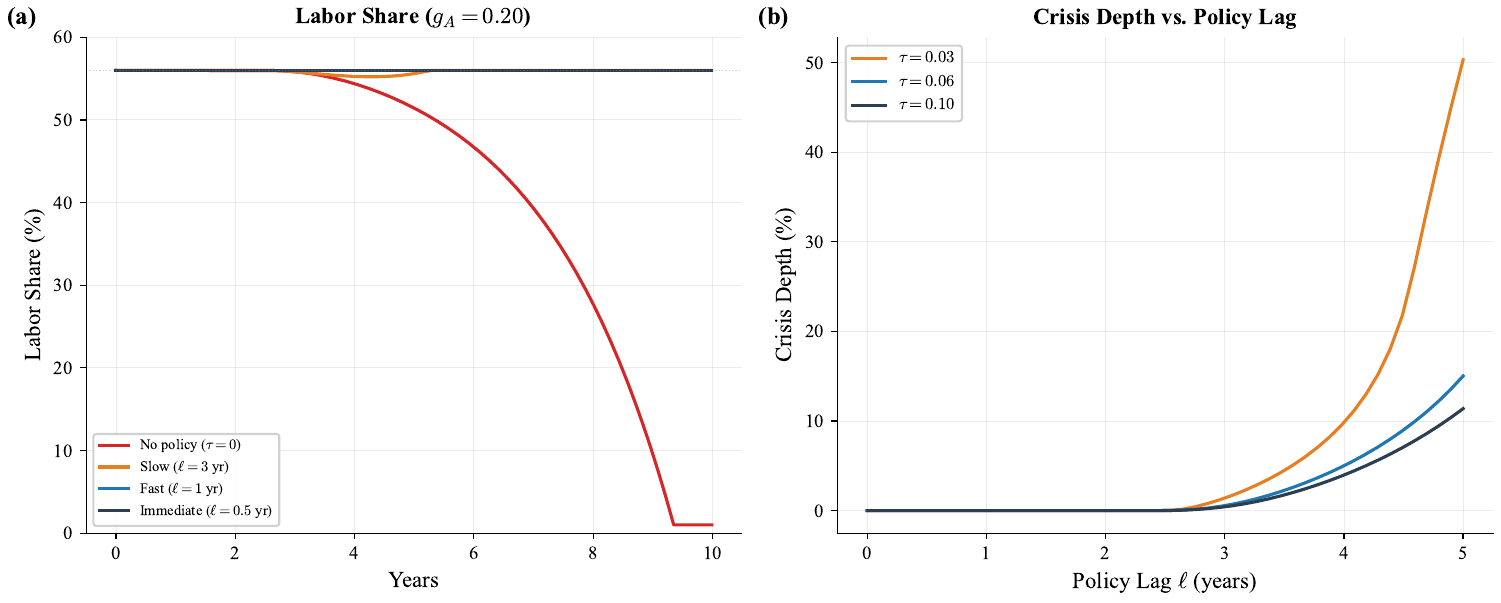}
\caption{Policy response under rapid AI adoption ($g_A = 0.20$). (a) Labor share trajectories under four policy scenarios. (b) Crisis depth as a function of policy lag $\ell$ for three transfer magnitudes.}
\label{fig:policy_sim}
\end{figure}

Figure~\ref{fig:policy_sim} confirms the central result of Proposition~\ref{prop:policy}: the depth of the crisis is a policy variable, not a technological inevitability. Under rapid AI adoption, an immediate and large fiscal response ($\ell = 0.5$ years, $\tau = 0.10$) stabilizes the labor share above 40\%, while no policy response allows it to collapse toward zero. Panel (b) shows that crisis depth is roughly linear in the policy lag for a given transfer magnitude, with a steeper slope for smaller transfers. This underscores the urgency of preemptive policy design: once the explosive regime is entered, the required transfer magnitude to stabilize the system increases rapidly.

\paragraph{Consumption shock decomposition.} Under a uniform 10\% income shock, the top quintile---which accounts for 59\% of consumption and has the highest AI exposure (0.60)---contributes 3.54 of the total 3.92 pp consumption decline, illustrating the amplification mechanism of Proposition~\ref{prop:concentration}. Online Appendix~D provides a full decomposition by income quintile.

\paragraph{Historical recession comparison.} Table~\ref{tab:recession_comparison} places the model-implied consumption declines in context by comparing them to realized peak-to-trough consumption and GDP declines across seven post-war U.S.\ recessions. The rapid AI scenario ($g_A = 0.20$) produces consumption declines of 5--12\%, comparable in magnitude to the 2007--2009 Global Financial Crisis ($-5.0\%$) and substantially exceeding supply-side recessions (1973--75, 1980, 1981--82). Two features distinguish the AI scenario from historical precedent: (i) GDP effects are ambiguous because productivity gains may offset demand destruction in measured output (the Ghost GDP mechanism), creating a potentially novel macroeconomic configuration in which falling consumption coexists with stable or rising GDP; and (ii) the income-concentration channel is more acute than in any prior recession, where income losses were distributed across quintiles rather than concentrated in the top quintile.

\begin{table}[H]
\centering
\caption{Historical U.S.\ Recession Comparison: Consumption Declines}
\label{tab:recession_comparison}
\small
\begin{tabular}{@{}llccp{5cm}@{}}
\toprule
\textbf{Recession} & \textbf{Name} & \textbf{Real PCE} & \textbf{Real GDP} & \textbf{Primary} \\
 & & \textbf{Decline (\%)} & \textbf{Decline (\%)} & \textbf{Mechanism} \\
\midrule
1973Q4--1975Q1 & Oil shock & -2.4 & -1.8 & Stagflation; supply-driven \\
1980Q1--1980Q3 & Volcker I & -2.2 & -1.3 & Monetary tightening \\
1981Q3--1982Q4 & Volcker II & -1.5 & -2.7 & Monetary tightening \\
1990Q3--1991Q1 & S\&L crisis & -0.8 & -1.4 & Financial; real estate \\
2001Q1--2001Q4 & Dot-com & -0.3 & 0.2 & Investment-led; mild \\
2007Q4--2009Q2 & GFC & -5.0 & -4.0 & Financial; housing \\
2020Q1--2020Q2 & COVID-19 & -17.1 & -9.1 & Lockdown; V-shaped \\
\addlinespace
\midrule
\multicolumn{5}{@{}l}{\textit{Model-implied AI scenarios (10-year cumulative, no policy intervention):}} \\
\addlinespace
AI Rapid ($g_A = 0.20$) & Stress test & $-5$ to $-12$ & Ambiguous$^\dagger$ & Demand-side; displacement \\
AI Extreme ($g_A = 0.40$) & Tail risk & $-15$ to $-30$ & Ambiguous$^\dagger$ & Demand-side; displacement \\
\bottomrule
\multicolumn{5}{@{}p{14cm}@{}}{\scriptsize \textit{Notes:} Historical data from BEA NIPA tables. $^\dagger$AI scenarios produce ambiguous GDP effects because productivity gains may offset demand destruction in measured output (Ghost GDP mechanism). The GFC is the closest historical analog under the rapid AI scenario.} \\
\end{tabular}
\end{table}

\section{Testable Predictions and Early Warning Indicators}
\label{sec:predictions}

The model yields eleven testable hypotheses, each with an observable indicator, data source, and explicit falsification condition. We present these in Table~\ref{tab:predictions}.

\begin{table}[!htbp]
\centering
\caption{Testable Predictions and Early Warning Indicators}
\label{tab:predictions}
\scriptsize
\renewcommand{\arraystretch}{0.9}
\begin{tabular}{@{}p{0.4cm}p{2.8cm}p{2.6cm}p{2.0cm}p{3.0cm}@{}}
\toprule
& \textbf{Predicted Pattern} & \textbf{Observable Indicator} & \textbf{Data Source} & \textbf{Falsification Condition} \\
\midrule
H1 & White-collar unemployment rises or wages stagnate in AI-exposed occupations & Information-sector unemployment; BLS wage percentiles for AI-exposed vs.\ control occupations & BLS CPS, OEWS; JOLTS; Indeed Hiring Lab & AI-exposed occupations show \textit{faster} wage growth or \textit{lower} unemployment than controls through 2028 \\
\addlinespace
H2 & SaaS revenue compression: seat-based pricing collapses, ARR growth decelerates & Median SaaS net revenue retention; new ARR per public SaaS company & Public SaaS earnings; KeyBanc SaaS survey; Bessemer Cloud Index & SaaS net retention remains above 110\% and seat-based pricing persists \\
\addlinespace
H3 & Credential inflation in AI-exposed roles & Credential requirements for stable-task roles & Lightcast/ Burning Glass & Requirements \textit{decline} in AI-exposed roles \\
\addlinespace
H4 & Consumption decline concentrated in upper-income cohorts & Consumer spending by income quintile; retail sales in high-income ZIP codes & BEA PCE; credit card spending data (Facteus); Census retail & Consumption declines are uniform across quintiles or concentrated in lower quintiles \\
\addlinespace
H5 & Private credit stress in software/tech portfolio; public--private mark gap widens & Default rates in PE-backed software loans; CLO spreads; public--private NAV gap & Moody's; Preqin; SEC Form PF; BIS & PE-backed software defaults below historical base rates and gap within 200 bps through 2028 \\
\addlinespace
H6 & Monetary velocity declines despite positive GDP growth & $V = \text{GDP}/M2$; divergence between GDP growth and personal income growth & FRED; BEA NIPA & Velocity increases or remains stable while GDP grows \\
\addlinespace
H7 & AI capability growth rate $g_A$ exceeds the reinstatement-dominated threshold & Frontier model benchmark scores (METR, MMLU, SWE-bench); agentic task completion rates; enterprise AI deployment surveys & METR; Epoch AI; Census ABS; McKinsey AI surveys & Capability benchmarks plateau or adoption surveys show $<$20\% enterprise deployment through 2028 \\
\addlinespace
H8 & AI-exposed sectors show price deflation insufficient to offset nominal wage compression (Ghost GDP dominates deflation channel) & Sector-level PPI/CPI in AI-affected categories relative to nominal wage growth in same sectors & BLS PPI/CPI; BEA industry accounts & AI-exposed sectors show price deflation \textit{exceeding} nominal wage compression, raising real incomes \\
\addlinespace
H9 & Reinstatement fails to keep pace with displacement & Net new job categories; AI-complementary postings growth & BLS OES; O*NET; job postings & AI-complementary occupations absorb workers at historical automation rates \\
\addlinespace
H10 & Payments: interchange compression as AI routes to lower-cost rails & Stablecoin settlement; merchant acceptance; card revenue/txn & Visa/MC; Chainalysis; Fed payments & Stablecoin $<$5\% of card volume; interchange stable through 2028 \\
\addlinespace
H11 & Prime mortgage stress in AI-exposed metros & 90+ day delinquency for 780+ FICO in tech hubs; HELOC utilization & NY Fed CCP; CoreLogic; FHFA & High-FICO delinquency in tech metros at or below national average \\
\bottomrule
\end{tabular}
\end{table}

The hypotheses are ordered by expected observability---H1 and H2 should show the earliest signals. H6 (velocity decline) is the most aggregated prediction and may take longest to manifest clearly. H7 directly monitors the key input parameter $g_A$ (AI capability growth), providing an upstream signal before downstream effects manifest. H8 and H9 monitor the two key \textit{competing mechanisms} (deflation/real-income effects and task reinstatement) that could attenuate or negate the crisis pathway. H10 and H11 monitor the two financial transmission channels---payments intermediation and prime mortgage stress---that require sector-specific tracking distinct from the aggregate indicators \citep{bis2024privatecredit, nyfed2025mortgage}. The falsification conditions are designed to be decisive: if H8 or H9's falsification conditions are met, the corresponding competing mechanism is likely strong enough to prevent crisis-magnitude displacement.

The cross-sectional regressions in Section~\ref{subsec:occevidence} already provide early evidence consistent with H1: AI-exposed occupations experienced significantly lower wage growth between 2019 and 2023, consistent with the wage compression channel. As the 2024--2028 data become available, these regressions can be extended into a panel setting with difference-in-differences identification, using the timing of generative AI deployment as a natural experiment.

\section{Evidence Assessment}
\label{sec:evidence}

We assess the current evidentiary basis for each mechanism, distinguishing between claims that are empirically established, directionally supported, and speculative. Table~\ref{tab:evidence} summarizes.

\begin{table}[!htbp]
\centering
\caption{Evidence Assessment Matrix}
\label{tab:evidence}
\scriptsize
\renewcommand{\arraystretch}{0.9}
\begin{tabular}{@{}p{2.2cm}p{2.8cm}p{2.8cm}p{2.5cm}@{}}
\toprule
\textbf{Mechanism} & \textbf{Established} & \textbf{Directional} & \textbf{Speculative} \\
\midrule
AI productivity \& slowdowns & 14--56\% task-level gains; experienced devs slower with AI in field & Gains extend to complex tasks; overhead limits net effect & Economy-wide acceleration \\
\addlinespace
Displacement spiral \& reinstatement & Labor share $\downarrow$19\%; AI-exposed wages compressed; new tasks historically offset automation & Agentic tools perform with limited oversight; AI-complementary roles emerging & Self-reinforcing macro feedback; reinstatement dominates \\
\addlinespace
Ghost GDP \& intermediation & Velocity $\downarrow$36\% since 1997; friction reduction is core AI capability & GDP--income gap widening; SaaS pricing pressure & Velocity acceleration; sector-wide margin collapse \\
\addlinespace
Consumption \& deflation & Top 20\% = 47--65\% of spending; inference costs $\downarrow$30\%/yr & High-income workers most AI-exposed; cloud prices falling & Concentrated shock; deflation offsets compression \\
\addlinespace
Credit \& fragility & Private credit \$2.5T; mortgages \$13T; leverage/opacity flagged by BIS & Software LBOs assume growth; multiple stressors could trigger repricing & Systemic cascading defaults; AI as unique trigger \\
\bottomrule
\end{tabular}
\end{table}

Several observations emerge. First, the micro-level evidence for AI productivity gains is strong but \textit{mixed}: controlled experiments show large speedups, while field studies document integration overhead and slowdowns for experienced workers \citep{metr2025slowdown}. This disparity motivates the diffusion wedge (Assumption~\ref{ass:capability}) and is itself an uncertain parameter. Second, the ``evaluation gap'' identified by the 2026 International AI Safety Report \citep{aisafetyreport2026} means the speed of capability improvement---the key parameter $g_A$---is itself uncertain. Third, the reinstatement channel (Remark~\ref{rem:reinstatement}) is a genuine competing mechanism: historically, new task creation has roughly offset automation, and it is an open empirical question whether generative AI breaks this pattern. Fourth, the financial plumbing that would transmit an AI-driven labor market shock into systemic stress is well-documented, but much of this fragility is \textit{pre-existing}: private credit leverage, insurer exposure to alternative assets, and valuation opacity are concerns independent of AI \citep{bis2024privatecredit, naic2024privatecredit}. Our stress test posits AI displacement as one potential trigger for repricing, not as the sole source of financial fragility. Fifth, the occupation-level regressions in Section~\ref{subsec:occevidence} shift the displacement mechanism from ``speculative'' toward ``directionally supported''---the wage compression pattern is already visible in the data, even if the macro-financial feedback loop has not yet activated.

\section{Policy Toolkit}
\label{sec:policy}

\subsection{Why Standard Monetary Policy Is Insufficient}
\label{subsec:monetary}

The displacement spiral and Ghost GDP mechanisms imply that standard monetary policy---interest rate adjustment and quantitative easing---is necessary but insufficient. Rate cuts address the \textit{financial conditions} channel (credit availability, asset prices, debt service costs) and can stabilize demand in the short run, but they do not address the \textit{structural displacement} channel \citep{cook2026ai}. This echoes the secular stagnation literature \citep{summers2015demand, eggertsson2019log}: when the problem is insufficient demand driven by income distribution, monetary policy's demand-side effects are attenuated. We emphasize ``insufficient,'' not ``useless'': monetary easing remains valuable for preventing financial amplification (credit crunches, liquidity spirals), and the credit transmission mechanism (Proposition~\ref{prop:credit}) can be partially mitigated by maintaining loose financial conditions. The constraint is that monetary policy alone cannot restore the labor share or create new tasks.

\subsection{Fiscal and Transfer Instruments}
\label{subsec:fiscal}

The binding constraint is the transfer rate $\tau_t$ in Proposition~\ref{prop:velocity}. Effective fiscal response requires: (i) \textit{Progressive AI/compute taxation}: a tax on AI inference compute, analogous to carbon pricing, that funds direct transfers to displaced workers; (ii) \textit{Expanded unemployment insurance}: extending duration and benefit levels for structurally displaced workers, recognizing that AI displacement may not be temporary; (iii) \textit{Sovereign AI dividend fund}: a public equity claim on returns from AI infrastructure, modeled on sovereign wealth funds, creating a direct channel from AI productivity gains to household income \citep[cf.][]{blanchard2019public}.

\subsection{Financial Stability Instruments}
\label{subsec:financial}

The credit transmission mechanism (Proposition~\ref{prop:credit}) motivates: (i) \textit{AI-augmented stress tests} incorporating scenarios of AI-driven revenue shocks in technology and services portfolios; (ii) \textit{Private credit transparency}: mandatory mark-to-market reporting with shorter lag for private credit funds with insurance company limited partners---U.S.\ life insurers held an estimated \$600--800 billion in private credit and alternative assets by 2024, with concentration in a small number of PE-affiliated insurers (Apollo/Athene, Brookfield/American Equity, KKR/Global Atlantic) that each hold \$50--200 billion in alternative assets \citep{naic2024privatecredit}; (iii) \textit{Concentration monitoring}: regulatory monitoring of the PE-insurer nexus for correlated exposure, particularly where private letter ratings may understate risk relative to public market benchmarks.

\subsection{Labor Market Instruments}
\label{subsec:labor}

(i) \textit{Shortened work week}: distributing available work across more workers, slowing the rate of labor share decline at the source; (ii) \textit{Retraining investment}: public investment in skills that complement AI---judgment, social intelligence, physical-world tasks \citep{deming2017growing}; (iii) \textit{Hiring incentives}: tax credits for firms that maintain or expand human employment.

\subsection{AI Governance and Deployment Speed}
\label{subsec:governance}

The preceding instruments address the \textit{consequences} of displacement. A complementary policy lever operates on the \textit{cause}: governing AI deployment speed to keep $g_A$ and $\kappa$ below the explosive threshold $g_A^*(\rho)$ in Proposition~\ref{prop:spiral}. Three governance channels are operative:

(i) \textit{Mandatory evaluation and staged deployment}: requiring capability evaluations and safety testing before deployment of frontier models in high-impact domains. The EU AI Act imposes phased GPAI obligations with compliance deadlines through 2027 \citep{euaiact2024}; major labs have adopted voluntary frontier-risk frameworks (OpenAI's Preparedness Framework, Anthropic's Responsible Scaling Policy, DeepMind's Frontier Safety Framework) \citep{aisafetyreport2026}. These frameworks directly constrain $\kappa$ by imposing deployment lags.

(ii) \textit{Compute governance}: monitoring and potentially licensing large-scale AI training runs, analogous to dual-use export controls. Compute thresholds provide a tractable regulatory hook because training compute is measurable, concentrated in a small number of providers, and correlated with capability \citep{aisafetyreport2026}. This channel affects $g_A$ at the source.

(iii) \textit{Sector-specific deployment restrictions}: phased rollout requirements in sectors where rapid displacement poses systemic risk (e.g., financial services, healthcare administration, legal services). Such restrictions slow sectoral diffusion ($\kappa_j$ for sector $j$) while allowing deployment in lower-risk domains.

The governance landscape is evolving rapidly but unevenly. U.S.\ policy has been volatile: Executive Order 14110 established safety directives in 2023 but was subsequently rescinded, with later orders emphasizing innovation acceleration \citep{whitehouse2025ai}. This volatility introduces regulatory uncertainty into the diffusion parameter $\kappa$---itself a source of fragility, as firms making hiring and investment decisions face an unpredictable regulatory environment. The key insight for our model is that \textit{governance instruments directly affect the parameters that determine which regime the economy enters}: binding evaluation requirements raise the effective $g_A^*(\rho)$ threshold by slowing deployment, while deregulation lowers it.

\subsection{The Policy Response Function}
\label{subsec:policyresponse}

\begin{proposition}[Policy Response Function]
\label{prop:policy}
The depth of the macro-financial crisis is:
\begin{equation}
\mathcal{D} \propto \max\left(0, \; \Delta s_L - \tau(t - \ell)\right),
\label{eq:policy}
\end{equation}
where $\Delta s_L$ is the cumulative labor share decline, $\tau(\cdot)$ is the fiscal transfer function, and $\ell$ is the policy response lag. Crisis depth is:
\begin{enumerate}[label=(\roman*)]
\item Increasing in the speed of AI adoption ($g_A$, which drives $\Delta s_L$).
\item Decreasing in the magnitude of fiscal redistribution ($\tau$).
\item Increasing in the policy lag ($\ell$).
\end{enumerate}
The crisis is fully averted if and only if $\tau(t - \ell) \geq \Delta s_L(t)$ for all $t$---that is, fiscal transfers keep pace with labor share decline in real time, accounting for the implementation lag.
\end{proposition}

\begin{proof}
From Proposition~\ref{prop:velocity}, the velocity decline (and associated demand shortfall) is proportional to $\Delta s_L - \tau$. The policy lag $\ell$ means that transfers at time $t$ respond to conditions at time $t - \ell$. If $\Delta s_L$ is accelerating (as in the explosive regime of Proposition~\ref{prop:spiral}), the lag means $\tau(t-\ell) < \Delta s_L(t)$ even if the eventual transfer is adequate, generating a transient crisis whose depth depends on the gap. Parts (i)--(iii) follow from differentiation of~\eqref{eq:policy}.
\end{proof}

The calibrated simulations in Figure~\ref{fig:policy_sim} quantify this proposition: under rapid AI adoption ($g_A = 0.20$), a policy lag of more than 2 years with transfer magnitude $\tau = 0.03$ produces a labor share decline exceeding 40\%---a crisis of unprecedented depth. The same AI adoption rate with a 0.5-year lag and $\tau = 0.10$ stabilizes the labor share above 40\%. The policy variable, not the technology variable, determines the outcome.

\paragraph{Implementation lags by instrument type.} The single parameter $\ell$ in Proposition~\ref{prop:policy} abstracts over substantial heterogeneity in implementation speed across policy instruments:
\begin{itemize}[nosep]
\item \textit{Monetary policy} ($\ell_m \approx$ 0.1--0.5 years): interest rate adjustments and quantitative easing can be deployed within weeks to months, but address only the financial conditions channel.
\item \textit{Automatic fiscal stabilizers} ($\ell_a \approx$ 0--0.5 years): existing unemployment insurance and progressive taxation respond automatically with near-zero lag, but are calibrated for cyclical, not structural, displacement.
\item \textit{Discretionary fiscal transfers} ($\ell_f \approx$ 1--3 years): new transfer programs, tax changes, and sovereign AI dividend funds require legislative action, implementation infrastructure, and political consensus.
\item \textit{Labor market restructuring} ($\ell_r \approx$ 3--10 years): retraining programs, work-week reforms, and the creation of new institutional frameworks for human-AI complementarity operate on the slowest timescale.
\end{itemize}
The effective policy lag in~\eqref{eq:policy} is a weighted average across instrument types, with the weights reflecting the political economy of response. If AI displacement is rapid ($g_A > 0.20$), only monetary policy and automatic stabilizers operate within the relevant time frame; discretionary transfers and labor market restructuring arrive too late to prevent the initial crisis, though they can affect its duration and depth. This creates a structural mismatch: the instruments best suited to address the \textit{cause} of displacement (labor market instruments) have the longest implementation lag, while the instruments with shortest lag (monetary policy) address only the \textit{symptoms}.

\section{Conclusion}
\label{sec:conclusion}

This paper develops a formal stress test for the macroeconomic consequences of rapid AI adoption, grounded in empirical evidence. We formalize three mechanisms---the displacement spiral (with a competing reinstatement channel), Ghost GDP, and intermediation collapse---derive the conditions under which each is self-limiting versus explosive, and discipline the analysis with U.S.\ macroeconomic data and occupation-level regressions.

The empirical analysis makes three contributions. First, FRED time series document the secular preconditions---a 19\% labor share decline, a 36\% velocity decline, and a widening productivity--compensation gap---on which the AI-driven mechanisms would operate. Second, cross-sectional regressions show that AI-exposed occupations already experienced significantly lower wage growth between 2019 and 2023, consistent with the early stages of the displacement spiral operating through wage compression, though causal identification remains limited. Third, calibrated simulations show that the interaction of the three mechanisms through the consumption concentration amplifier can generate crisis-level demand shortfalls under rapid AI adoption, but that the depth of any crisis is governed by both the reinstatement rate and the policy response function.

We emphasize four points. First, this is a stress test, not a prediction. The mechanisms we formalize have explicit conditions for activation, and we provide falsification criteria for each. Second, the model nests the benign outcome: if task reinstatement is strong, diffusion is slow, or institutions adapt quickly, the reinstatement-dominated regime (Proposition~\ref{prop:spiral}(i)) obtains and the crisis does not materialize. Third, the depth of any resulting crisis is a policy variable, not a technological inevitability. Proposition~\ref{prop:policy} and the simulations in Figure~\ref{fig:policy_sim} show that sufficiently fast and large fiscal redistribution can fully offset the demand destruction from AI displacement. Fourth, the standard macroeconomic indicators that policymakers rely on---GDP growth, productivity, corporate earnings---may be precisely the wrong signals during a rapid AI transition, because Ghost GDP drives a wedge between these indicators and the consumption-relevant economy.

The practical contribution is an observable early warning framework (Table~\ref{tab:predictions}) that can be monitored in real time. The eleven hypotheses include not only signals of crisis activation (H1--H6) but also direct monitoring of the key input parameter (H7: AI capability growth rate), explicit monitoring of the two competing mechanisms---AI-driven deflation (H8) and task reinstatement (H9)---that could attenuate or prevent the crisis, and sector-specific tracking of the two financial transmission channels---payments intermediation (H10) and prime mortgage stress in AI-exposed geographies (H11). If the crisis indicators remain quiescent and the competing mechanism indicators are strong through 2028, the crisis pathway is unlikely to be operative at the magnitudes we model. If multiple crisis indicators signal simultaneously while H8 and H9's falsification conditions are not met, the framework provides a structured basis for preemptive policy action.

\paragraph{What would change our assessment.} We conclude by specifying, at the meta level, what evidence would lead us to conclude that the stress-test scenario is not operative. The scenario would be substantially disconfirmed if, by 2028: (1) AI-exposed occupations show wage growth at or above the economy-wide average (falsifying the displacement channel); (2) new AI-complementary job categories emerge at rates matching or exceeding historical automation episodes (confirming the reinstatement-dominated regime); (3) AI-driven price deflation in major consumption categories exceeds nominal wage compression, raising real incomes for the median household (confirming the deflation channel); and (4) private credit and mortgage markets show no elevation in default rates or spread widening in AI-exposed sectors. If conditions (1)--(4) hold jointly, the mechanisms we formalize are not operating at crisis-relevant magnitudes, and the stress test should be revised downward. Conversely, if conditions (1) and (4) are violated simultaneously---AI-exposed workers experience wage declines \textit{and} credit markets begin repricing AI-exposed portfolios---the crisis pathway is actively materializing, and the policy response framework in Section~\ref{sec:policy} becomes operationally relevant.

\bigskip

\singlespacing
\bibliographystyle{plainnat}
\bibliography{bibliography}

\end{document}